\documentclass{article}

\usepackage{times}
\usepackage{graphicx}
\usepackage{subfigure} 
\usepackage{natbib}
\usepackage[colorlinks=true]{hyperref}

\usepackage[accepted]{styles/icml2016} 

\usepackage{url}
\usepackage{amsmath}
\usepackage{mathrsfs}
\usepackage{booktabs}
\usepackage{amssymb}
\usepackage{amsthm}
\usepackage{enumerate}
\newtheorem{remark}{Remark}
\newtheorem{example}{Example}
\newtheorem{theorem}{Theorem}
\newtheorem{task}{Task}
\newtheorem{proposition}{Proposition}
\newtheorem{definition}{Definition}
\usepackage[ruled,vlined]{algorithm2e}
\SetKwRepeat{Do}{do}{while}
\usepackage{soul}

\begin{document} 

\twocolumn[
\icmltitle{Cross-Graph Learning of Multi-Relational Associations}
\icmlauthor{Hanxiao Liu}{hanxiaol@cs.cmu.edu}
\icmlauthor{Yiming Yang}{yiming@cs.cmu.edu}
\icmladdress{Carnegie Mellon University, Pittsburgh, PA 15213, USA}
\icmlkeywords{Multi-relational Learning, ICML}
\vskip 0.3in
]

\begin{abstract} 
Cross-graph Relational Learning (CGRL) refers to the problem of predicting the strengths or labels of multi-relational tuples of heterogeneous object types,
through the joint inference over multiple graphs which specify the internal connections among each type of objects.
CGRL is an open challenge in machine learning due to the daunting number of all possible tuples to deal with when the numbers of nodes in multiple graphs are large,
and because the labeled training instances are extremely sparse as typical.
Existing methods such as tensor factorization or tensor-kernel machines do not work well because of the lack of convex formulation for the optimization of CGRL models, the poor scalability of the algorithms in handling combinatorial numbers of tuples, and/or the non-transductive nature of the learning methods which limits their ability to leverage unlabeled data in training.
This paper proposes a novel framework which formulates CGRL as  a convex optimization problem, enables transductive learning using both labeled and unlabeled tuples, and offers a scalable algorithm that guarantees the optimal solution and enjoys a linear time complexity with respect to the sizes of input graphs.
In our experiments with a subset of DBLP publication records and an Enzyme multi-source dataset, the proposed method successfully scaled to the large cross-graph inference problem, and outperformed other representative approaches significantly. 
\end{abstract} 

\section{Introduction}
\label{sec:intro}
Many important problems in multi-source relational learning could be cast as joint learning over multiple graphs about how heterogeneous types of objects interact with each other. In literature data analysis, for example, publication records provide rich information about how authors collaborate with each other in a co-authoring graph, how papers are linked in citation networks, how keywords are related via ontology, and so on. The challenging question is about how to combine such heterogeneous information in individual graphs for the labeling or scoring of the multi-relational associations in tuples like \texttt{(author,paper,keyword)},  given some observed instances of such tuples as the labeled training set.  Automated labeling or scoring of unobserved tuples allows us to discover who have been active in the literature on what areas of research, and to predict who would become influential in which areas in the future. 
In protein data analysis, as another example, a graph of proteins with pairwise sequence similarities is often jointly studied with a graph of chemical compounds with their structural similarities for the discovery of interesting patterns in \texttt{(compound,protein)} pairs. We call the prediction problem in both examples \textit{cross-graph learning of multi-relational associations}, or simply \textit{cross-graph relational learning} (CGRL), where the multi-relational associations are defined by the tuples of heterogeneous types of objects, and each object type has its own graph with type-specific relational structure as a part of the provided data.  The task is to predict the labels or the scores of unobserved multi-relational tuples, conditioned on a relatively small set of labeled instances.

CGRL is an open challenge in machine learning for several reasons. Firstly, the number of multi-relational tuples grows combinatorially in the numbers of individual graphs and the number of nodes in each graph.
How to make cross-graph inference computationally tractable for large graphs is a tough challenge. Secondly, how to combine the internal structures or relations in individual graphs for joint inference in a principled manner is an open question.  Thirdly, supervised information (labeled instances) is typically extremely sparse in CGRL due to the very large number of all possible combinations of heterogeneous objects in individual graphs.  Consequently, the success of cross-graph learning crucially depends on effectively leveraging the massively available unlabeled tuples (and the latent relations among them) in addition to the labeled training data. In other words, how to make the learning transductive is crucial for the true success of CGRL.  Research on transdcutive CGRL has been quite limited, to our knowledge. 

Existing approaches in CGRL or CGRL-related areas can be outlined as those using tensors or graph-regularized tensors, and kernel machines that combine multiple kernels.

Tensor methods have been commonly used for combining multi-source evidence of the interactions among multiple types of objects \cite{nickel2011three, rendle2009learning, kolda2009tensor} as the combined evidence can be naturally represented as tuples. However, most of the tensor methods do not explicitly model the internal graph structure for each type of objects, although some of those methods implicitly leverage such information via graph-based regularization terms in their objective function that encourage similar objects within each graph to share similar latent factors
\cite{narita2012tensor, cai2011graph}.
A major weakness in such tensor methods is the lack of convexity in their models, which leads to ill-posed optimization problems particularly in high-order scenarios. It has also been observed that tensor factorization models suffer from label-sparsity issue, which is typically severe in CGRL. 

Kernel machines have been widely studied for supervised classifiers,
where a kernel matrix corresponds to a similarity graph among a single type of objects.
Multiple kernels can be combined, 
for example,
by taking the tensor product of each individual kernel matrix,
which results in a desired kernel matrix among cross-graph multi-relational tuples. 
The idea has been explored in relational learning combined with SVMs \cite{ben2005kernel}, perceptions \cite{basilico2004unifying} or Gaussian process \cite{yu2008gaussian} for two types of objects and is generalizable to the multi-type scenario of CGRL.
Although being generic,
the complexity of such kernel-based methods grows exponentially in the number of individual kernels (graphs) and the size of each individual graph.  As a result, kernel machines suffer from poor scalability in general.
In addition, kernel machines are purely supervised (not for transductive learning), i.e., they cannot leverage the massive number of available non-observed tuples induced from individual graphs and the latent connections among them.  Those limitations make existing kernel methods less powerful for solving the CGRL problem in large scale and under severely data-sparse conditions.


In this paper, we propose a novel framework for CGRL which can be characterized as follows:
(i) It uses graph products to map heterogeneous sources of information and the link structures in individual graphs onto a single \emph{homogeneous} graph;
(ii) It provides a convex formulation and approximation of the CGRL problem that ensure robust optimization and efficient computation; and
(iii) It enables transductive learning in the form of label propagation over the induced homogeneous graph so that the massively available non-observed tuples and the latent connections among them can play an important role in effectively addressing the label-sparsity issue. 

The proposed framework is most related to \cite{liu2015bipartite},
where the authors formulated graph products for learning the edges of a bipartite graph.
Our new framework is fundamentally different in two aspects.
First, our new formulation and algorithms allow the number of individual graphs to be greater than two, 
while method in \cite{liu2015bipartite} is only applicable to two graphs.  
Secondly, the
algorithms in \cite{liu2015bipartite} suffer from cubic complexity over the graphs sizes
(quadratic by using a non-convex approximation),
while our new algorithm enjoys both the convexity of the formulation and the low time complexity which is linear over the graph sizes.

The paper is organized as follows:
Section \ref{sec:proposed} shows how cross-graph multi-relations can be embedded into the vertex space of a homogeneous graph.
Section \ref{sec:approximation} describes how efficient label propagation among multi-relations can be carried out in such space with approximation.
We discuss our optimization algorithm in Section \ref{sec:opt}
and provide empirical evaluations over real-world datasets in Section \ref{sec:exp}.

\section{The Proposed Method}
\begin{figure*}[ht]
    \centering
    \resizebox{1.3cm}{!}{$\mathscr{P}\bigg($}
        \begin{minipage}[c]{0.11\linewidth}
            \centering
            $\underbrace{\includegraphics[width=0.9\linewidth]{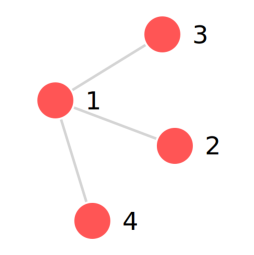}}_{\text{\large $G^{(1)}$}}$
        \end{minipage}
        ,
        \begin{minipage}[c]{0.09\linewidth}
            \centering
            \vspace{0.35cm}
            $\underbrace{\includegraphics[width=0.9\linewidth]{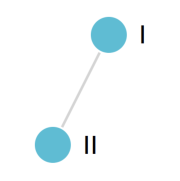}}_{\text{\large $G^{(2)}$}}$
        \end{minipage} 
        ,
        \begin{minipage}[c]{0.11\linewidth}
            \centering
            $\underbrace{\includegraphics[width=0.9\linewidth]{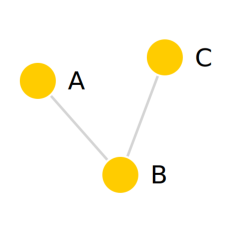} }_{\text{\large $G^{(3)}$}}$
        \end{minipage}
        \
    \resizebox{1.25cm}{!}{$\bigg)=$}
    \begin{minipage}[c]{0.375\linewidth}
        \centering
        \includegraphics[width=0.75\linewidth]{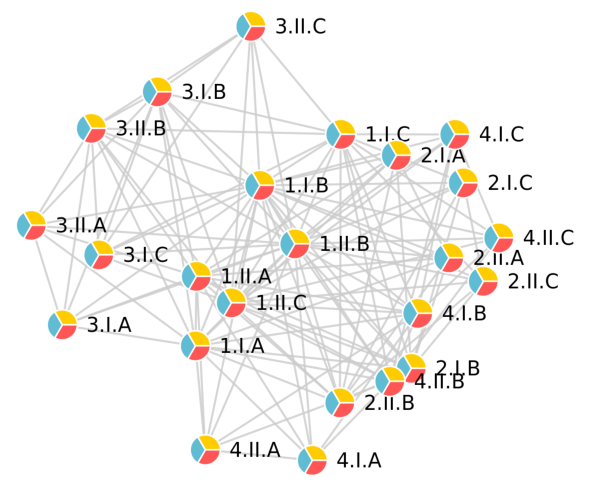} 
    \end{minipage} 
    \caption{Graph product of $G^{(1)}$, $G^{(2)}$ and $G^{(3)}$.
        Each vertex in the resulting graph $\mathscr{P}\big(G^{(1)}$, $G^{(2)}$, $G^{(3)}\big)$ corresponds to a multi-relation across the original graphs. E.g., vertex \texttt{3.II.B} in $\mathscr{P}$ corresponds to multi-relation \texttt{(3,II,B)} across $G^{(1)}$, $G^{(2)}$ and $G^{(3)}$.}
    \label{fig:gp-example}
\end{figure*}

\label{sec:proposed}
We introduce our notation in \ref{sec:notations}
and the notion of graph product (GP) in \ref{sec:gp}.
We then narrow down to a specific GP family with desirable computational properties in \ref{sec:gp},
and finally propose our GP-based optimization objective in \ref{sec:obj}.

\subsection{Notations}
\label{sec:notations}
We are given $J$ heterogeneous graphs
where the $j$-th graph contains $n_j$ vertices and is associated with an adjacency matrix $G^{(j)} \in \mathbb{R}^{n_j \times n_j}$.
We use $i_j$ to index the $i_j$-th vertex of graph $j$,
and use a tuple $(i_1, \dots, i_J)$ to index each multi-relation across the $J$ graphs.
The system predictions over all possible $\prod_{j=1}^J n_j$ multi-relations is summarized in an order-$J$ tensor $f \in \mathbb{R}^{{n_1} \times \dots \times n_J}$,
where $f_{i_1,i_2,\dots,i_J}$ corresponds to the prediction about tuple $(i_1, \dots, i_J)$.

Denote by $\otimes$ the Kronecker (Tensor) product.
We use $\bigotimes_{j=1}^{J} x_j$ (or simply $\bigotimes_j x_j$) as the shorthand for $x_1 \otimes \dots \otimes x_J$.
Denote by $\times_j$ the $j$-mode product between tensors.
We refer the readers to \cite{kolda2009tensor} for a thorough introduction about tensor mode product.

\subsection{Graph Product}
\label{sec:gp}
In a nutshell, graph product (GP)
\footnote{
    While traditional GP only applies to two graphs, we generalize it to the case of multiple graphs
    (Section \ref{sec:sgp}).
}
is a mapping from each cross-graph multi-relation to each vertex in a new graph $\mathscr{P}$,
whose edges encode similarities among the multi-relations (illustrated in Fig.\ \ref{fig:gp-example}).
A desirable property of GP is
it provides a natural reduction from the original multi-relational learning problem over \emph{heterogeneous} information sources (Task \ref{task:1})
to an equivalent graph-based learning problem over a \emph{homogeneous} graph (Task \ref{task:2}).

\begin{task}
    \label{task:1}
    Given $J$ graphs $G^{(1)}, \dots, G^{(J)}$ with a small set of labeled multi-relations
    $\mathcal{O} = \{(i_1, \dots, i_J)\}$,
    predict labels of the unlabeled multi-relations.
\end{task}

\begin{task}
    \label{task:2}
    Given the product graph $\mathscr{P}\big(G^{(1)}, \dots, G^{(J)}\big)$ with a small set of labeled vertices
    $\mathcal{O} = \{(i_1, \dots, i_J)\}$,
    predict labels of its unlabeled vertices.
\end{task}

\subsection{Spectral Graph Product}
\label{sec:sgp}
We define a parametric family of GP operators named the spectral graph product (SGP),
which is of particular interest as
it subsumes the well-known Tensor GP and Cartesian GP (Table \ref{tab:example}),
is well behaved (Theorem \ref{thm:commutation})
and allows efficient optimization routines (Section \ref{sec:approximation}).

Let $\lambda^{(j)}_{i_j}$ and $v^{(j)}_{i_j}$ be the $i_j$-th eigenvalue and eigenvector for the graph $j$,
respectively.
We construct SGP by defining the eigensystem of its adjacency matrix
based on the provided $J$ heterogeneous eigensystems of $G^{(1)},\dots, G^{(J)}$.

\begin{definition}
    \label{def:sgp}
    The SGP of $G^{(1)}, \dots, G^{(J)}$ is a graph consisting of $\prod_j n_j$ vertices,
    with its adjacency matrix $\mathscr{P}_\kappa := \mathscr{P}_\kappa\left( G^{(1)}, \dots, G^{(J)} \right)$ defined by the following eigensystem
    \begin{align}
        \Big\{ 
            \kappa\big(\lambda^{(1)}_{i_1}, \dots, \lambda^{(J)}_{i_J} \big),
            \bigotimes_j {v^{(j)}_{i_j}}
        \Big\}_{i_1, \dots, i_J}
    \end{align}
    where $\kappa$ is a pre-specified nonnegative nondecreasing function over $\lambda_{i_j}^{(j)}, \forall j = 1,2,\dots, J$.
\end{definition}
In other words,
the $(i_1,\dots,i_J)$-th eigenvalue of $\mathscr{P}_\kappa$ is defined by
    coupling the $\lambda^{(1)}_{i_1}, \dots, \lambda^{(J)}_{i_J}$ with function $\kappa$,
    and the $(i_1,\dots,i_J)$-th eigenvector of $\mathscr{P}_\kappa$
    is defined by coupling $v^{(1)}_{i_1}, \dots, v^{(J)}_{i_J}$
    via tensor (outer) product.

\begin{remark}
    If each individual $\big\{v^{(j)}_{i_j}\big\}_{{i_j}=1}^{n_j}$ forms an orthogonal basis
    in $\mathbb{R}^{n_j}$, $\forall j \in 1,\dots,J$,
    then $\big\{\bigotimes_j {v^{(j)}_{i_j}}\big\}_{i_1, \dots, i_J}$ forms an orthogonal basis in $\mathbb{R}^{\prod_{j=1}^J n_j}$.
\end{remark}

In the following example we introduce two special kinds of SGPs,
assuming $J=2$ for brevity.
Higher-order cases are later summarized in Table \ref{tab:example}.

\begin{example}
    Tensor GP defines
    $\kappa(\lambda_{i_1}, \lambda_{i_2}) = \lambda_{i_1} \lambda_{i_2}$,
    and is equivalent to Kronecker product:
    $
        \mathscr{P}_{\text{Tensor}}\big(G^{(1)}, G^{(2)}\big)
        = \sum_{i_1, i_2} (\lambda_{i_1} \lambda_{i_2}) \big(v^{(1)}_{i_1} \otimes v^{(2)}_{i_2}\big) \big(v^{(1)}_{i_1} \otimes v^{(2)}_{i_2}\big)^\top
        \equiv G^{(1)} \otimes G^{(2)}
    $.

    Cartesian GP defines
    $\kappa(\lambda_{i_1}, \lambda_{i_2}) = \lambda_{i_1} + \lambda_{i_2}$,
    and is equivalent to the Kronecker sum:
    $
        \mathscr{P}_{\text{Cartesian}}\big(G^{(1)}, G^{(2)}\big)
        = \sum_{i_1, i_2} (\lambda_{i_1} + \lambda_{i_2}) \big(v^{(1)}_{i_1} \otimes v^{(2)}_{i_2}\big) \big(v^{(1)}_{i_1} \otimes v^{(2)}_{i_2}\big)^\top
        \equiv G^{(1)} \oplus G^{(2)}
    $.
        \begin{table}[ht]
        \small
            \centering
    \resizebox{\columnwidth}{!}{
            \begin{tabular}{r|cc}
                \toprule
                SGP Type & $\kappa\big( \lambda_{i_1}^{(1)}, \cdots, \lambda_{i_J}^{(J)} \big)$ & $\left[\mathscr{P}_\kappa\right]_{(i_1, \cdots i_J), (i'_1, \cdots i'_J)}$ \\
                \midrule
                Tensor & $\prod_j \lambda_{i_j}^{(j)}$ & $\prod_j G^{(j)}_{i_j, i'_j}$ \\
                Cartesian & $\sum_j \lambda_{i_j}^{(j)}$ & $\sum_j G^{(j)}_{i_j, i'_j} \prod_{j'\not=j} \delta_{i_{j'} = i'_{j'}}$ \\
                \bottomrule
            \end{tabular}
        }
            \caption{Tensor GP and Cartesian GP in higher-orders.}
            \label{tab:example}
        \end{table}
\end{example}

While Tensor GP and Cartesian GP provide mechanisms to associate multiple graphs
in a multiplicative/additive manner,
more complex cross-graph association patterns can be modeled by specifying $\kappa$.
E.g.,
$
    \kappa\left( \lambda_{i_1}, \lambda_{i_2}, \lambda_{i_3} \right)
    = \lambda_{i_1} \lambda_{i_2} + \lambda_{i_2} \lambda_{i_3} + \lambda_{i_3} \lambda_{i_1}
$
indicates pairwise associations are allowed among three graphs,
but no triple-wise association is allowed as term $\lambda_{i_1} \lambda_{i_2} \lambda_{i_3}$ is not involved.
Including higher order polynomials in $\kappa$
amounts to incorporating higher-order associations among the graphs,
which can be achieved by simply exponentiating $\kappa$.

Since what the product graph $\mathscr{P}$ offers is essentially a similarity measure among multi-relations,
shuffling the order of input graphs $G^{(1)}, \dots, G^{(J)}$ should not affect $\mathscr{P}$'s topological structure.
For SGP, this property is guaranteed by the following theorem:
\begin{theorem}[The Commutative Property]
    \label{thm:commutation}
    SGP is commutative (up to graph isomorphism) if $\kappa$ is commutative.
\end{theorem}
We omit the proof.
The theorem suggests the SGP family is well-behaved as long as $\kappa$ is commutative,
which is true for both Tensor and Cartesian GPs as both multiplication and addition operations are order-insensitive.

\subsection{Optimization Objective}
\label{sec:obj}
It is often more convenient to equivalently write tensor $f$ as a multi-linear map.
E.g., when $J = 2$,
tensor (matrix) $f \in \mathbb{R}^{n_1 \times n_2}$ defines a bilinear map
from $\mathbb{R}^{n_1} \times \mathbb{R}^{n_2}$ to $\mathbb{R}$ via $f(x_1, x_2) := x_1^\top f x_2$
and we have $f_{i_1,i_2} = f(e_{i_1}, e_{i_2})$.
Such equivalence is analogous to high-order cases
where $f$ defines a multi-linear map
from $\mathbb{R}^{n_1} \times \dots \times \mathbb{R}^{n_J}$ to $\mathbb{R}$.

To carry out transductive learning over $\mathscr{P}_\kappa$ (Task \ref{task:2}),
we inject the structure of the product graph into $f$ via a Gaussian random fields prior \cite{zhu2003semi}.
The negative log-likelihood of the prior $-\log p\left( f \mid \mathscr{P}_\kappa \right)$ is the same (up to constant) as the following squared semi-norm
\begin{align}
    \label{eq:semi-norm}
    \|f\|_{\mathscr{P}_\kappa}^2
    &=
    vec(f)^\top \mathscr{P}^{-1}_\kappa vec(f) \\
    \label{eq:semi-norm-sum}
    &=
    \sum_{i_1, i_2, \dots, i_J} 
    \frac{f\big( v^{(1)}_{i_1}, \dots, v^{(J)}_{i_J} \big)^2}{\kappa\big( \lambda_{i_1}^{(1)}, \dots, \lambda_{i_J}^{(J)} \big)}
\end{align}
Our optimization objective is therefore defined as
\begin{equation} \label{eq:obj}
    \min_{f \in \mathbb{R}^{n_1 \times \dots \times n_J}} \enskip \ell_{\mathcal{O}}\left(f\right)
    +
    \frac{\gamma}{2} \|f\|_{\mathscr{P}_\kappa}^2
\end{equation}
where $\ell_{\mathcal{O}}(\cdot)$ is a loss function to be defined later (Section \ref{sec:opt}),
$\mathcal{O}$ is the set of training tuples,
and $\gamma$ is a tuning parameter controlling the strength of graph regularization.

\section{Convex Approximation}
\label{sec:approximation}

\begin{figure*}
    \centering
    \includegraphics[width=0.85\linewidth]{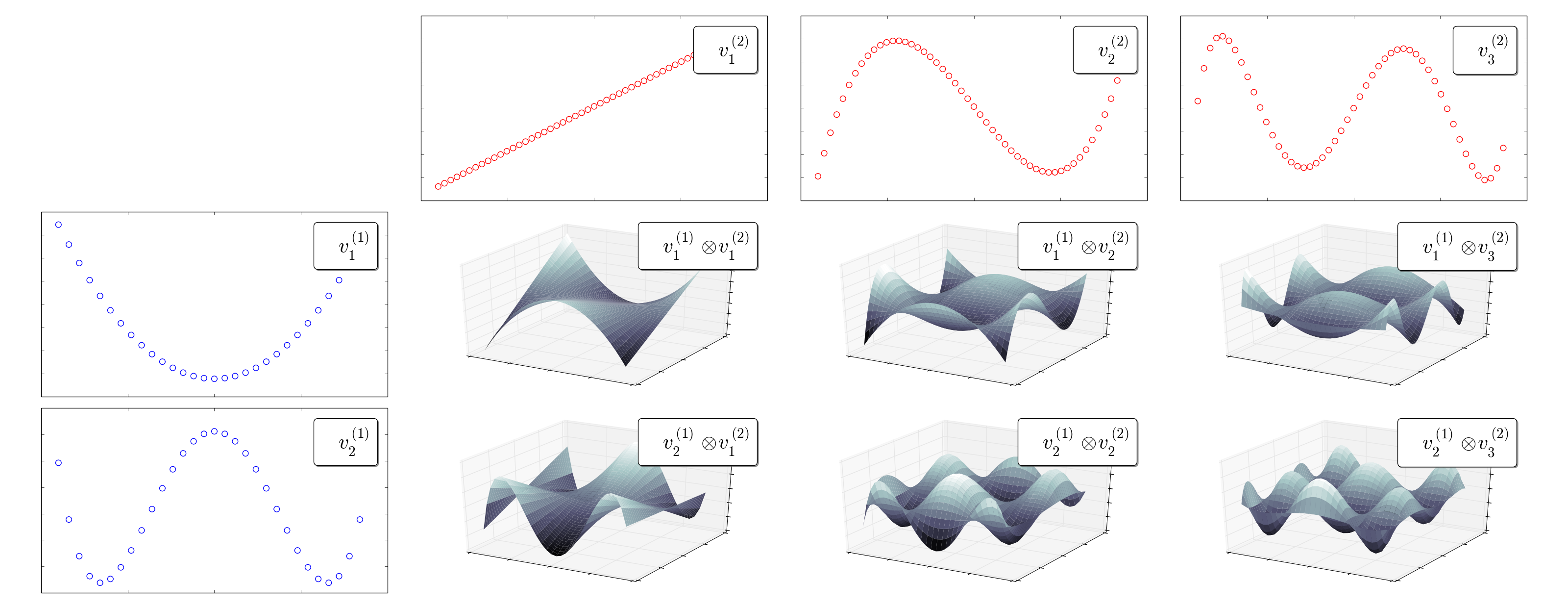}
    \caption{
        An illustration of the eigenvectors of $G^{(1)}$, $G^{(2)}$ and $\mathscr{P}\big( G^{(1)}, G^{(2)} \big)$.
        We plot leading nontrivial eigenvectors of $G^{(1)}$ and $G^{(2)}$ in blue and red curves, respectively,
        and plot the induced leading nontrivial eigenvectors of $\mathscr{P}\big( G^{(1)}, G^{(2)} \big)$ in 3D.
        If $G^{(1)}$ and $G^{(2)}$ are symmetrically normalized,
        their eigenvectors (corresponding to eigenvectors of the graph Laplacian) will be ordered by smoothness w.r.t.\ the graph structures.
        As a result, eigenvectors of $\mathscr{P}\big( G^{(1)}, G^{(2)} \big)$ will also be ordered by smoothness.
    }
    \label{fig:smooth}
\end{figure*}

The computational bottleneck for optimization \eqref{eq:obj} lies in evaluating $\|f\|_{\mathscr{P}_\kappa}^2$ and its first-order derivative,
due to the extremely large size of $\mathscr{P}_\kappa$.
In section \ref{sec:exact},
we first identify the computation bottleneck of using the exact formulation,
based on which we propose our convex approximation scheme in \ref{sec:tucker}
that reduces the time complexity of evaluating the semi-norm $\|f\|^2_{\mathscr{P}_\kappa}$ from $O\left(\big(\sum_j n_j\big) \big(\prod_j n_j\big) \right)$ to $O\big(\prod_j d_j\big)$,
where $d_j \ll n_j$ for $j = 1,\dots,J$.
\subsection{Complexity of the Exact Formulation} \label{sec:exact}
The brute-force evaluation of $\|f\|^2_{\mathscr{P}_\kappa}$ according to \eqref{eq:semi-norm-sum} costs 
$O\big( \big(\prod_j n_j\big)^2 \big)$,
as one has to evaluate $O\big( \prod_j n_j \big)$ terms inside the summation
where each term costs $O\big( \prod_j n_j \big)$.
However, redundancies exist and
the minimum complexity for the exact evaluation is given as follows
\begin{proposition}
    The exact evaluation of semi-norm $\|f\|_{\mathscr{P}_\kappa}$ takes $O\big(\big(\sum_j n_j\big) \big(\prod_j n_j\big) \big)$ flops.
\end{proposition}
\begin{proof}
    Notice that the collection of all numerators in \eqref{eq:semi-norm-sum}, namely $\big[f\big( v^{(1)}_{i_1}, \dots, v^{(J)}_{i_J} \big)\big]_{i_1, \cdots, i_J}$,
    is a tensor in $\mathbb{R}^{n_1 \times \dots \times n_J}$ that can be precomputed via
    \begin{equation} \label{eq:multiply}
        \big(\big(f \times_1 V^{(1)}\big) \times_2 V^{(2)}\big) \dots \times_J V^{(J)}
    \end{equation}
    where $\times_j$ stands for the $j$-mode product between a tensor in $\mathbb{R}^{n_1 \times \dots \times n_j \times \dots \times n_J}$
    and $V^{(j)} \in \mathbb{R}^{n_j \times n_j}$.
    The conclusion follows as the $j$-th mode product in \eqref{eq:multiply} takes $O\big( n_j \prod_j n_j \big)$ flops,
    and one has to do this for each $j = 1,\dots,J$.
    When $J=2$,
    \eqref{eq:multiply} reduces to the multiplication of three matrices ${V^{(1)}}^\top f V^{(2)}$ at the complexity of $O\left( (n_1 + n_2)n_1n_2 \right)$.
\end{proof}

\subsection{Approximation via Tucker Form}
\label{sec:tucker}
Equation \eqref{eq:multiply} implies the key for complexity reduction is to reduce the cost of the $j$-mode multiplications $\cdot \times_j V^{(j)}$.
Such multiplication costs $O\big( n_j \prod_j n_j \big)$ in general, but can be carried out more efficiently if $f$ is structured.

Our solution is twofold:
First, we include only the top-$d_j$ eigenvectors in $V^{(j)}$ for each graph $G^{(i)}$,
where $d_j \ll n_j$.
Hence each $V^{(j)}$ becomes a thin matrix in $\mathbb{R}^{n_j \times d_j}$.
Second,
we restrict tensor $f$ to be within the linear span of the top $\prod_{j=1}^J d_j$ eigenvectors of the product graph $\mathscr{P}_\kappa$
\begin{align} 
    f &= \sum_{k_1, \cdots, k_J = 1}^{d_1, \cdots, d_J} \alpha_{k_1, \cdots, k_J} \bigotimes_{j} v^{(j)}_{k_j} \label{eq:tucker} \\
    &= \alpha \times_1 V^{(1)} \times_2 V^{(2)} \times_3 \dots \times_J V^{(J)} 
\end{align}
The combination coefficients $\alpha \in \mathbb{R}^{d_1 \times \cdots \times d_J}$ is
known as the core tensor of Tucker decomposition.
In the case where $J=2$,
the above is equivalent to saying $f \in \mathbb{R}^{n_1 \times n_2}$
is a low-rank matrix parametrized by $\alpha \in \mathbb{R}^{d_1 \times d_2}$
such that $f = \sum_{k_1, k_2} \alpha_{k_1, k_2} v^{(1)}_{k_1} {v^{(2)}_{k_2}}^\top = V^{(1)} \alpha {V^{(2)}}^\top$.

Combining \eqref{eq:tucker} with the orthogonality property of eigenvectors
leads to the fact that $f\big( v_{k_1}^{(1)}, \dots, v_{k_J}^{(J)} \big) = \alpha_{k_1, \cdots, k_J}$.
To see this for $J=2$, notice
$f\big( v_{k_1}^{(1)}, v_{k_2}^{(2)} \big) = {v_{k_1}^{(1)}}^\top f  v_{k_1}^{(2)} = {v_{k_1}^{(1)}}^\top  V^{(1)} \alpha {V^{(2)}}^\top v_{k_1}^{(2)} = e_{k_1}^\top \alpha e_{k_2} = \alpha_{k_1,k_2}$.
Therefore the semi-norm in \eqref{eq:semi-norm} can be simplified as
\begin{equation}
    \label{eq:reg}
    \|f\|_{\mathscr{P}_\kappa}^2
    =
    \|\alpha\|_{\mathscr{P}_\kappa}^2
    =
    \sum_{k_1, \dots, k_J=1}^{d_1, \cdots, d_J} 
    \frac{\alpha^2_{k_1, \cdots, k_J}}{\kappa\big( \lambda_{k_1}^{(1)}, \dots, \lambda_{k_J}^{(J)} \big)}
\end{equation}

Comparing \eqref{eq:reg} with \eqref{eq:semi-norm-sum},
the number of inside-summation terms is reduced from $O\big( \prod_j n_j \big)$ to $O\big( \prod_j d_j \big)$ where $d_j \ll n_j$.
In addition, the cost for evaluating each term inside summation is reduced from $O\big( \prod_j n_j \big)$ to $O(1)$.

Denote by $V^{(j)}_{i_j} \in \mathbb{R}^{d_j}$ the $i_j$-th row of $V^{(j)}$,
we obtain the following optimization by replacing $f$ with $\alpha$ in \eqref{eq:obj}
\begin{equation} \label{eq:tucker-obj}
    \begin{aligned}
        \min_{\alpha \in \mathbb{R}^{d_1 \times \dots \times d_J}}
        &\ell_\mathcal{O} \left(f\right)
        + \frac{\gamma}{2} \|\alpha\|_{\mathscr{P}_\kappa}^2 \\
        \text{s.t.}\quad\enskip\ &f = \alpha \times_1 V^{(1)} \times_2 \dots \times_J V^{(J)}
    \end{aligned}
\end{equation}
Optimization above has intuitive interpretations.
In principle,
it is natural to emphasis bases in $f$ that are ``smooth'' w.r.t.\ the manifold structure of $\mathscr{P}_\kappa$,
and de-emphasis those that are ``nonsmooth'' in order to obtain a parsimonious hypothesis with strong generalization ability.
We claim this is exactly the role of regularizer \eqref{eq:reg}.
To see this,
note any nonsmooth basis $\bigotimes_j v_{k_j}^{(j)}$ of $\mathscr{P}_\kappa$ is likely to be associated with small a eigenvalue $\kappa\big( \lambda^{(1)}_{k_1},\dots, \lambda^{(J)}_{k_J}\big)$
(illustrated in Fig.\ \ref{fig:smooth}).
The conclusion follows by noticing that
$\alpha_{k_1,\dots,k_J}$ is essentially the activation strength of $\bigotimes_j v_{k_j}^{(j)}$ in $f$ (implied by \eqref{eq:tucker}),
and that \eqref{eq:reg} is going to give any $\alpha_{k_1,\dots,k_J}$ associated with a small $\kappa\big( \lambda^{(1)}_{k_1},\dots, \lambda^{(J)}_{k_J}\big)$ a stronger penalty.

\eqref{eq:tucker-obj} is a convex optimization problem over $\alpha$ with any convex $\ell_\mathcal{O}(\cdot)$.
Spectral approximation techniques for graph-based learning has been found successful in standard classification tasks \cite{fergus2009semi},
which are special cases under our framework when $J=1$.
We introduce this technique for multi-relational learning,
which is particularly desirable as the complexity reduction will be much more significant for high-order cases $(J >= 2)$.

While $f$ in \eqref{eq:tucker} is assumed to be in the Tucker form,
other low-rank tensor representation schemes are potentially applicable.
E.g.,
the Candecomp/Parafac (CP) form that further restricts $\alpha$ to be diagonal,
which is more aggressive but substantially less expressive.
The Tensor-Train decomposition \cite{oseledets2011tensor}
offers an alternative representation scheme in the middle of Tucker and CP,
but the resulting optimization problem will suffer from non-convexity.


\section{Optimization}
\label{sec:opt}

Let $\left( x \right)_+ = \max\left( 0, 1-x \right)$ be the shorthand for hinge loss.
We define $\ell_\mathcal{O}(f)$ to be the ranking $\ell_2$-hinge loss
\begin{equation} \label{eq:loss}
    \ell_\mathcal{O}(f)
    =
    \frac{\sum_{
    \text{
        \tiny
        $
        \begin{array}{c}
            (i_1, \dots, i_J) \in \mathcal{O} \\
            (i_1', \dots, i_J') \in \bar{\mathcal{O}}
        \end{array}
        $
}} \Big( f_{i_1 \dots i_J} - f_{i_1' \dots i_J'} \Big)^2_+}{|\mathcal{O}\times\bar{\mathcal{O}}|}
\end{equation}
where $\bar{\mathcal{O}}$ is the complement of $\mathcal{O}$ w.r.t.\ all possible multi-relations.
Eq.\ \eqref{eq:loss} encourages the valid tuples in our training set $\mathcal{O}$ to be ranked higher than those corrupted ones in $\bar{\mathcal{O}}$,
and is known to be a surrogate of AUC.


We use stochastic gradient descent for optimization as $|\mathcal{O}|$ is usually large.
In each iteration,
a random valid multirelation $(i_1, \dots, i_J)$ is uniformly drawn from $\mathcal{O}$,
a random corrupted multirelation $(i'_1, \dots, i'_J)$ is uniformly drawn from $\bar{\mathcal{O}}$.
The associated noisy gradient is computed as
\begin{align}
    \nabla_\alpha &= \frac{\partial \ell_\mathcal{O}}{\partial f} \left(
        \frac{\partial f_{i_1, \dots, i_J}}{\partial \alpha}
        -
        \frac{\partial f_{i'_1, \dots, i'_J}}{\partial \alpha}
    \right)
    + \gamma \alpha \oslash \kappa \label{eq:sgd-alpha}
\end{align}
where we abuse the notation by defining $\kappa \in \mathbb{R}^{d_1 \times \dots \times d_J}$,
$\kappa_{k_1,\dots,k_J} := \kappa\big(\lambda^{(1)}_{k_1}, \dots, \lambda^{(J)}_{k_J}\big)$;
$\oslash$ is the element-wise division between tensors.
The gradient w.r.t.\ $\alpha$ in \eqref{eq:sgd-alpha} is
\begin{align}
    \frac{\partial f_{i_1, \dots, i_J}}{\partial \alpha}
    &= \frac{\partial \big(\alpha \times_1 V^{(1)}_{i_1} \times_2 \dots \times_J V^{(J)}_{i_J}\big)}{\partial \alpha} \\
    &= \bigotimes_j V^{(j)}_{i_j} \quad \in \mathbb{R}^{d_1 \times \dots d_J} \label{eq:coupled}
\end{align}
Each SGD iteration costs $O\big( \prod_j d_j \big)$ flops,
which is independent from $n_1, n_2, \dots, n_J$.
After obtaining the solution $\hat{\alpha}(\kappa)$ of optimization \eqref{eq:tucker-obj} for any given SGP $\mathscr{P}_\kappa$,
our final predictions in $\hat{f}(\kappa)$ can be recovered via \eqref{eq:tucker}.

Following \cite{duchi2011adaptive}, we allow adaptive step sizes for each element in $\alpha$.
That is, in the $t$-th iteration we use
$
    \eta^{(t)}_{k_1, \dots, k_J} = \eta_0 \Big/ \Big[\sum_{\tau=1}^t { {\nabla_\alpha}_{k_1, \dots, k_J}^{(\tau)}}^2\Big]^{\frac{1}{2}}
$
as the step size for $\alpha_{k_1, \dots, k_J}$,
where $\big\{{\nabla_\alpha}_{k_1, \dots, k_J}^{(\tau)}\big\}_{\tau=0}^t$ are historical gradients associated with $\alpha_{k_1, \dots, k_J}$
and $\eta_0$ is the initial step size (set to be $1$).
The strategy is particularly efficient with highly redundant gradients,
which is our case where the gradient is a regularized rank-2 tensor, according to \eqref{eq:sgd-alpha} and \eqref{eq:coupled}.

In practice (especially for large $J$),
the computation cost of tensor operations involving $\bigotimes_{j=1}^J V^{(j)}_{i_j} \in \mathbb{R}^{d_1, \dots, d_J}$
is not ignorable even if $d_1, d_2, \dots, d_J$ are small.
Fortunately, such medium-sized tensor operations in our algorithm are highly parallelable over GPU. 
The pseudocode for our optimization algorithm is summarized in Alg.\ \ref{alg:code}.

\begin{algorithm}[t]
    \label{alg:code}
    \caption{\small{Transductive Learning over Product Graph (TOP)}}
    \ForEach{$j \in 1, \dots, J$}
    {
        $\big\{ v^{(j)}_k, \lambda^{(j)}_k \big\}_{k=1}^{d_j} \gets \textsc{Approx\_Eigen}(G^{(j)})$\;
    }
    \ForEach{$\left( k_1, \dots, k_J \right) \in [d_1] \times \dots [d_J]$}
    {
        $\kappa_{k_1, \dots, k_J} \gets \kappa(\lambda^{(1)}_{k_1}, \dots, \lambda^{(J)}_{k_J})$\;
    }
    $\alpha \gets 0$, $Z \gets 0$\;
    \While{not converge}
    {
        $(i_1, \dots, i_J) \stackrel{uni}{\sim} \mathcal{O}$,\enskip\,$(i'_1, \dots, i'_J) \stackrel{uni}{\sim} \bar{\mathcal{O}}$\;
        $f_{i_1,\dots,i_J} \gets \alpha \times_1 V_{i_1}^{(1)} \times_2 \dots \times_J V_{i_J}^{(J)}$\;
        $f_{i'_1,\dots,i'_J} \gets \alpha \times_1 V_{i'_1}^{(1)} \times_2 \dots \times_J V_{i'_J}^{(J)}$\;
        $\delta = f_{i_1,\dots,i_J} - f_{i'_1,\dots,i'_J}$\;
        \If{$\delta < 1$}
        {
            $\nabla_\alpha \gets 2(\delta-1) \Big(\bigotimes_{j} V^{(j)}_{i_j} - \bigotimes_{j} V^{(j)}_{i'_j}\Big) + \gamma \alpha \oslash \kappa$\;
        }
        \Else
        {
            $\nabla_\alpha \gets \gamma \alpha \oslash \kappa$\;
        }
        $Z \gets Z + \nabla_\alpha^{\odot 2}$\;
        $\alpha \gets \alpha - \eta_0 Z^{\odot - \frac{1}{2}} \odot \nabla_\alpha$\;
    }
    \Return $\alpha$
\end{algorithm}

\section{Experiments}
\label{sec:exp}

\subsection{Datasets}
We evaluate our method on real-world data in two different domains:
the Enzyme dataset \cite{yamanishi2008prediction} for compound-protein interaction 
and the DBLP dataset of scientific publication records.
Fig.\ \ref{fig:schema} illustrates their heterogeneous objects and relational structures.
\begin{figure*}[ht]
    \centering
    \includegraphics[width=0.25\linewidth]{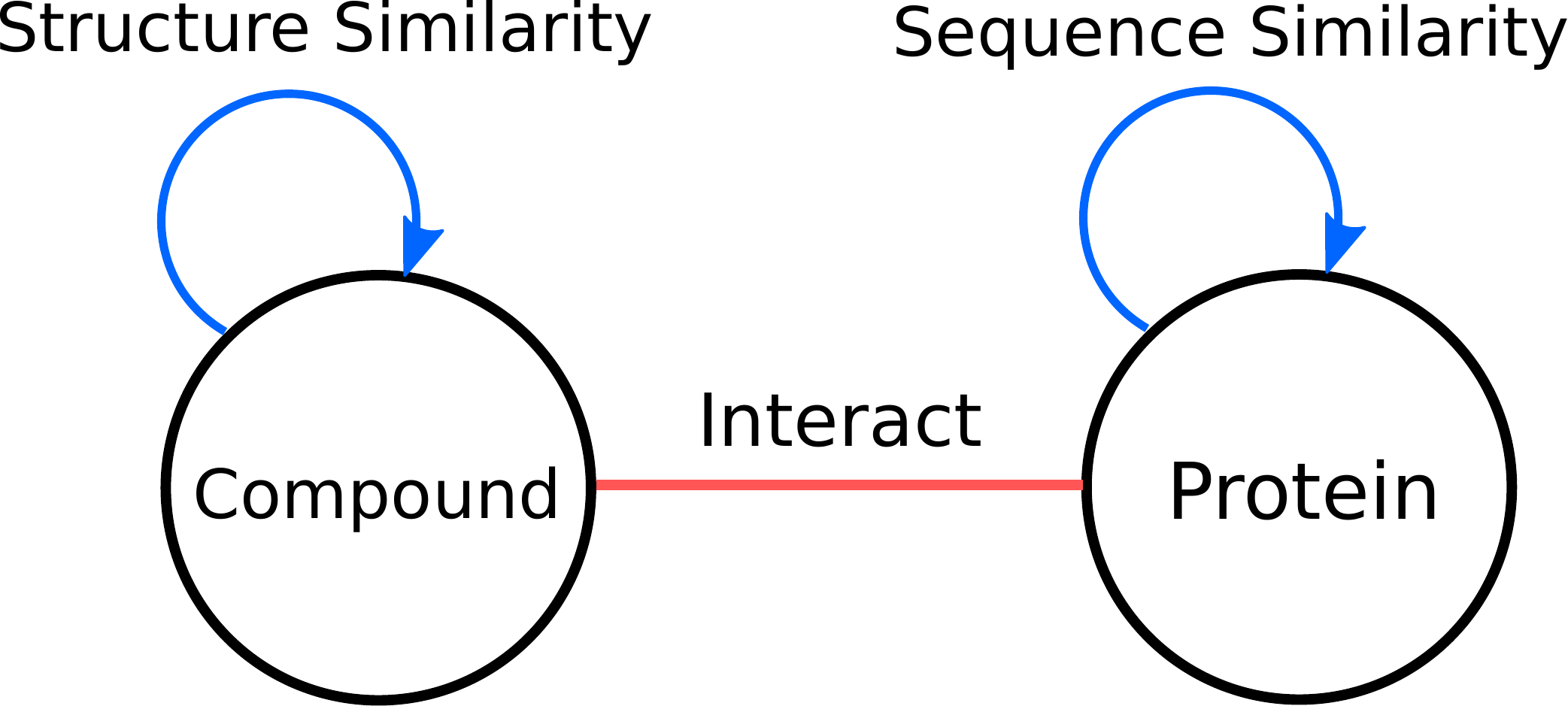}
    \hspace{4em}
    \includegraphics[width=0.32\linewidth]{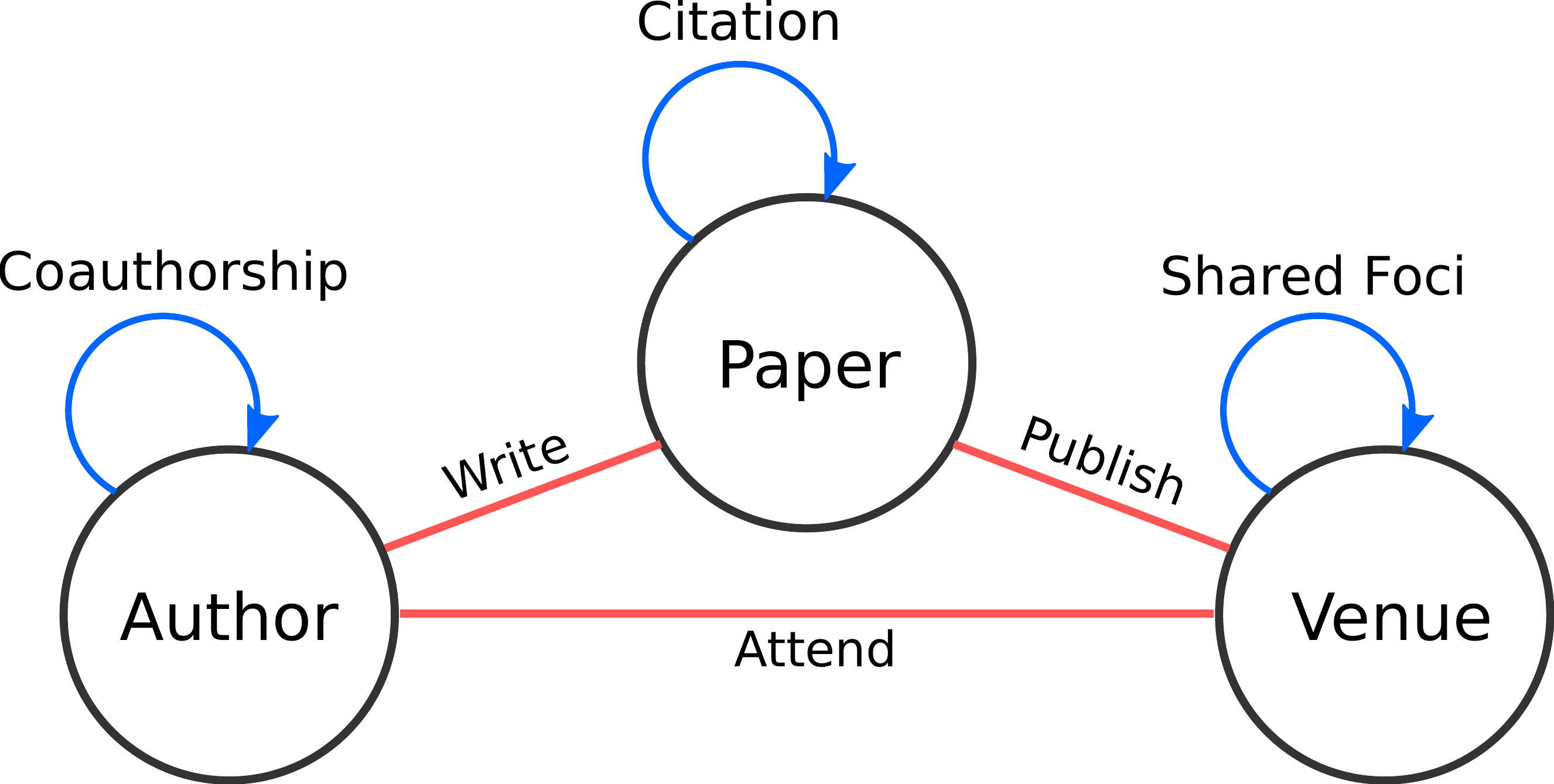}
    \caption{The heterogeneous types of objects (the circles) and the relational structures in the Enzyme (left) and DBLP (right) data sets.
        The blue edges represent the within-graph relations and the red edges represent the cross-graph interactions. The corresponding tuples in Enzyme is in the form of \texttt{(Compound,Protein)},
    and in DBLP is in the form of \texttt{(Author,Paper,Venue)}.}
    \label{fig:schema}
\end{figure*}

The Enzyme dataset
has been used for modeling and predicting drug-target interactions,
which contains a graph of 445 chemical compounds (drugs) and a graph of 664 proteins (targets).
The prediction task is to label the unknown compound-protein interactions based on both the graph structures and a small set of 2,926 known interactions.
The graph of compounds is constructed based on the SIMCOMP score \cite{hattori2003heuristics},
and the graph of proteins is constructed based on the normalized SmithWaterman score \cite{smith1981identification}.
While both graphs are provided in the dense form,
we converted them into sparse $k$NN graphs where each vertex is connected with its top 1\% neighbors.

As for the DBLP dataset,
we use a subset of 34,340 DBLP publication records in the domain of Artificial Intelligence \cite{tang2008arnetminer}, from which 3 graphs are constructed as:
\begin{itemize}
    \item For the author graph ($G^{(1)}$) we draw an edge between two authors if they have coauthored an overlapping set of papers, and remove the isolated authors using a DFS algorithm. We then obtain a symmetric $k$NN graph by connecting each author with her top $0.5\%$ nearest neighbors using the count of co-authored papers as the proximity measure. The resulting graph has 5,517 vertices with 17 links per vertex on average.
    \item For the paper graph ($G^{(2)}$) we connect two papers if both of them cite another paper, or are cited by another paper.
        Like $G^{(1)}$, we remove isolated papers using DFS and construct a symmetric 0.5\%-NN graph.
        To measure the similarity of any given pair of papers,
        we represent each paper as a bag-of-citations and compute their cosine similarity.
        The resulted graph has 11,879 vertices and has an average degree of 50.
    \item For the venue graph ($G^{(3)}$) we connect two venues if they share similar research focus.
        The venue-venue similarity is measured by the total number of cross-citations in between,
        normalized by the size of the two venues involved.
        The symmetric venue graph has 22 vertices and an average degree of 7.
\end{itemize}
Tuples in the form of \texttt{(Author,Paper,Venue)} are extracted from the publication records,
and there are 15,514 tuples (cross-graph interactions) after preprocessing.

\begin{figure}
    \centering
    \includegraphics[width=0.95\linewidth]{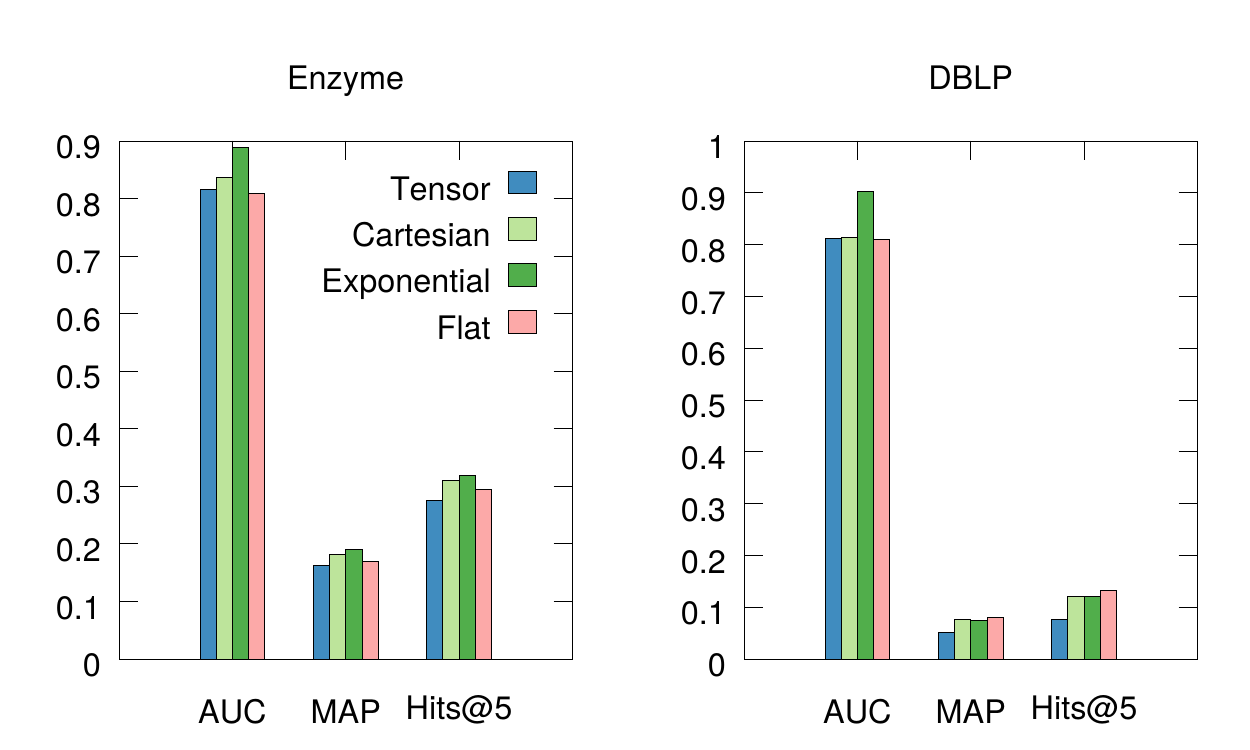}
    \caption{Performance of TOP with different SGPs.}
    \label{fig:TOP on Enzyme and DBLP}
\end{figure}

\subsection{Methods for Comparison}

\begin{itemize}
    \item Transductive Learning over Product Graph (\textbf{TOP}). \\
        The proposed method.
        We explore the following 
        $\kappa$'s for parametrizing the spectral graph product. \\
        {
            \small
            \\
            \centering
            \begin{tabular}{cll}
                Name & $\kappa(x,y)$ $(J=2)$ & $\kappa(x,y,z)$ $(J=3)$ \\
                \midrule
                Tensor & $xy$ & $xyz$ \\
                Cartesian & $x+y$ & $x+y+z$ \\
                Exponential & $e^{x+y}$ & $e^{xy + yz + xz}$ \\
                Flat & $1$ & 1 \\
            \end{tabular}
        }
    \item Tensor Factorization (\textbf{TF}) and Graph-regularized TF (\textbf{GRTF}).
        In TF we factorize $f \in \mathbb{R}^{n_1 \times \dots \times n_J}$
        as a set of dimensionality-reduced latent factors
        $C^{d_1, \times \dots \times d_J}$, $U_1^{n_1 \times d_1}, \dots, U_J \in \mathbb{R}^{n_J \times d_J}$.
        In GRTF,
        we further enhanced the traditional TF by adding graph regularizations to the objective function,
        which enforce the model to be aware of the context information in $G^{(j)}$'s
        \cite{narita2012tensor, cai2011graph};
    \item One-class Nearest Neighbor (\textbf{NN}).
        We score each tuple $\left(i_1, \dots, i_J\right)$ in the test set with
        $
            \hat{f}{(i_1, \dots, i_J)} = \mathrm{max}_{\left( i_1', \dots, i_J' \right) \in \mathcal{O}} \ \prod_{j=1}^J G_{i_j i'_j}
            $.
        That is,
        we assume the tuple-tuple similarity can be factorized as
        the product of vertex-level similarities across different graphs.
        We experimented with several other similarity measures and empirically found the multiplicative similarity leads to the best overall performance.
        Note it does not rely on the presence of any negative examples.
    \item Ranking Support Vector Machines \cite{joachims2002optimizing} (\textbf{RSVM}).
        For the task of completing the missing paper in \texttt{(Author,?,Venue)},
        we use a Learning-to-Rank strategy by treating \texttt{(Author,Venue)} as the query
        and \texttt{Paper} as the document to be retrieved.
        The query feature is constructed by concatenating the eigen-features of $\texttt{Author}$ and $\texttt{Venue}$,
        where we define the eigen-feature of vertex $i_j$ in graph $j$ as $V^{(j)}_{i_j} \in \mathbb{R}^{d_j}$.
        The feature for each query-document pair is obtained by taking the tensor product of the query feature and document eigen-feature.
    \item Low-rank Tensor Kernel Machines (\textbf{LTKM}).
        While traditional tensor-based kernel construction methods for tuples suffer from poor scalability. 
        We propose to speedup by replacing each individual kernel with its low-rank approximation before tensor product,
        leading to a low-rank kernel of tuples which allows more efficient optimization routines.
\end{itemize}

\begin{figure*}[ht]
    \centering
    \includegraphics[width=0.825\linewidth]{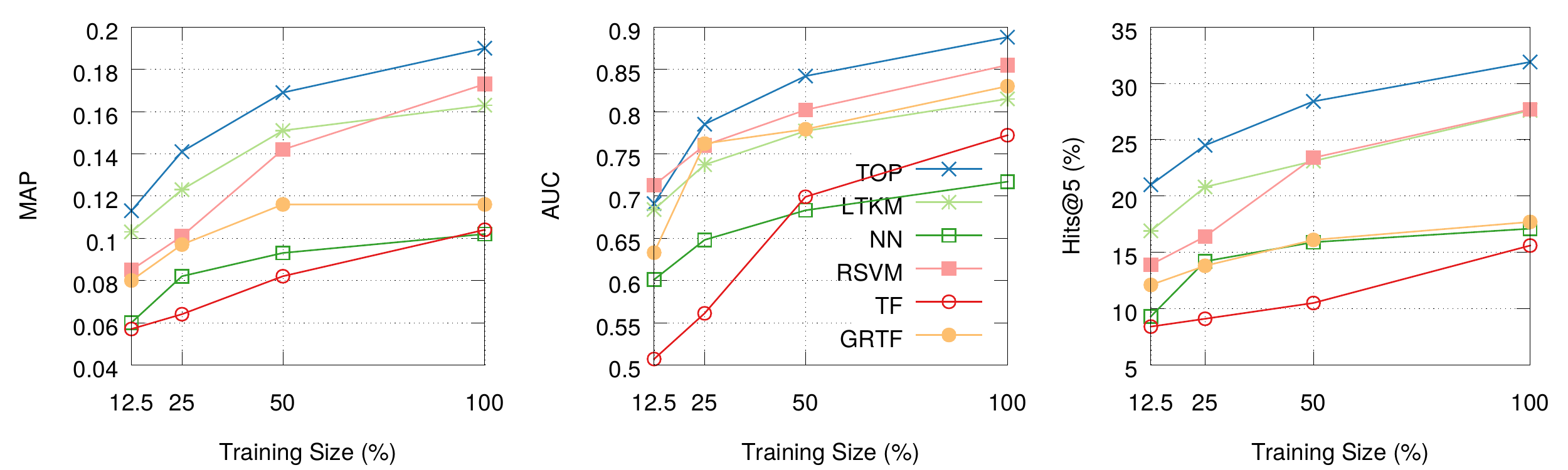}
    \caption{Test-set performance of different methods on Enzyme.}
    \label{fig:Methods on Enzyme}
\end{figure*}

\begin{figure*}[ht]
    \centering
    \includegraphics[width=0.825\linewidth]{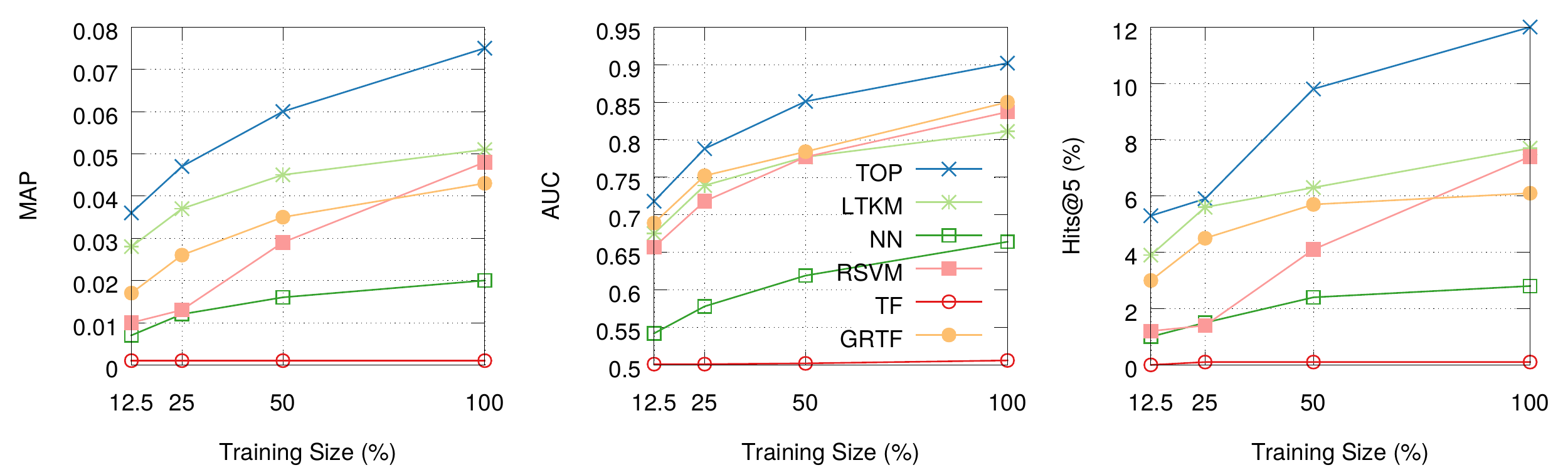}
    \caption{Test-set performance of different methods on DBLP.}
    \label{fig:Methods on DBLP}
\end{figure*}

For fair comparison,
loss functions for TF, GRTF, RSVM and LTKM
are set to be exactly the same as that for TOP, i.e. E.q.\ \eqref{eq:loss}.
All algorithms are trained using a mini-batched stochastic gradient descent.

We use the same eigensystems (eigenvectors and eigenvalues) of the $G^{(j)}$'s
as the input for TOP, RSVM and LTKM.
The number of top-eigenvalues/eigenvectors $d_j$ for graph $j$ is chosen
such that $\lambda^{(j)}_1, \dots ,\lambda^{(j)}_{d_j}$
approximately cover $80\%$ of the total spectral energy of $G^{(j)}$.
With respect to this criterion,
we choose $d_1=1,281$, $d_2 = 2,170$, $d_3 = 6$ for DBLP,
and $d_1=150$, $d_2=159$ for Enzyme.

\subsection{Experiment Setups}
For both datasets, we randomly sample one third of known interactions 
for training (denoted by $\mathcal{O}$), one third for validation and use the remaining ones for testing.
Known interactions in the test set,
denoted by $\mathcal{T}$, are treated as positive examples.
All tuples not in $\mathcal{T}$,
denoted by $\bar{\mathcal{T}}$, are treated as negative.
Tuples that are already in $\mathcal{O}$ are removed from $\bar{\mathcal{T}}$ to avoid misleading results \cite{bordes2013translating}.

We measure algorithm performance on Enzyme
based on the quality of inferred target proteins given each compound,
namely by the ability of completing $\texttt{(Compound,?)}$.
For DBLP,
the performance is measured by the quality of inferred papers given author and venue,
namely by the ability of completing $\texttt{(Author,?,Venue)}$.
We use Mean Average Prevision (MAP),
Area Under the Curve (AUC) and Hits at Top 5 (Hits@5) as our evaluation metrics.

\subsection{Results}

Fig.\ \ref{fig:TOP on Enzyme and DBLP} compares the results of TOP
with various parameterizations of the spectral graph product (SGP).
Among those, Exponential $\kappa$ works better on average.  

Figs.\ \ref{fig:Methods on Enzyme} and \ref{fig:Methods on DBLP} show the main results,
comparing TOP (with Exponential $\kappa$) with other representative baselines.
Clearly, TOP outperforms all the other methods on both datasets in all the evaluation metrics of MAP \footnote{MAP scores for random guessing are 0.014 on Enzyme and 0.00072 on DBLP, respectively.}, AUC and Hit@5.

\begin{figure}[h]
    \centering
    \includegraphics[width=0.75\linewidth]{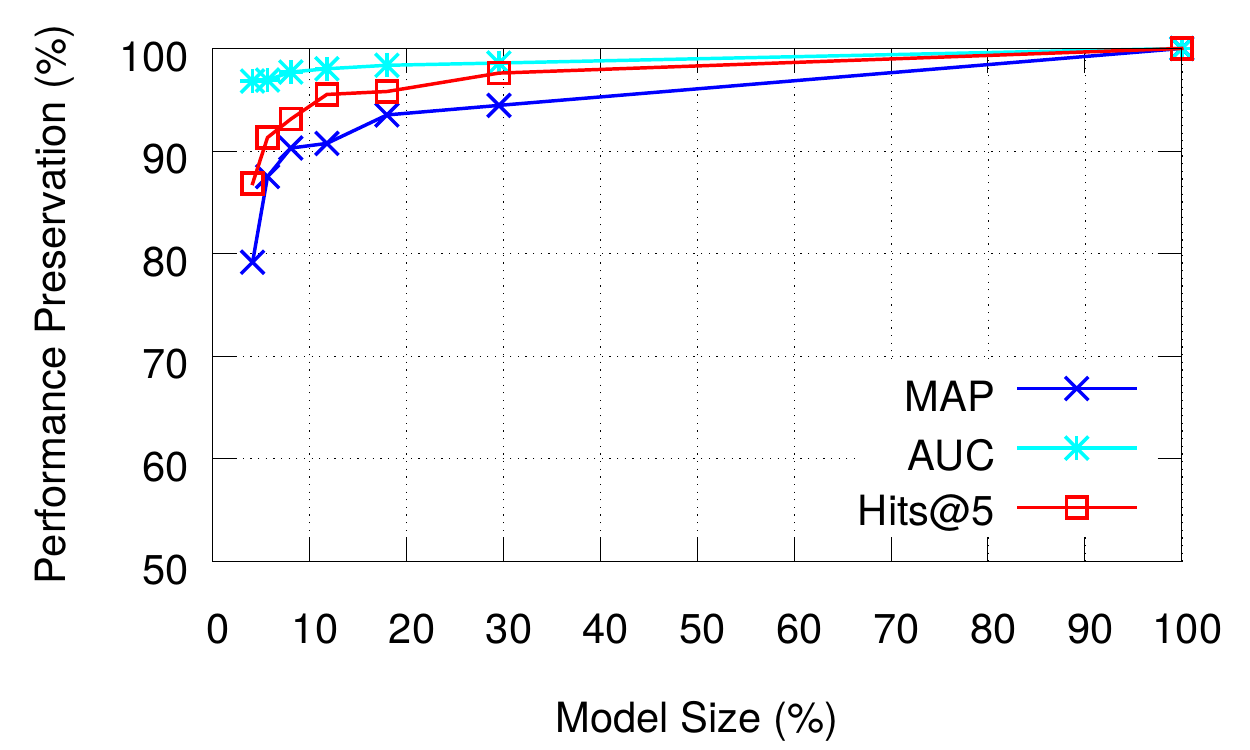}
    \caption{Performance of TOP v.s.\ model size on Enzyme.}
    \label{fig:eigen k}
\end{figure}

Fig.\ \ref{fig:eigen k} shows the performance curves of TOP on Enzyme over different model sizes (by varying the $d_j$'s).
With a relatively small model size compared with using the full spectrum, TOP's performance converges to the optimal point.

\section{Concluding Remarks}
The paper presents a novel convex optimization framework for transductive CGRL and a scalable algorithmic solution with guaranteed global optimum
and a time complexity that does not depend on the sizes of input graphs.
Our experiments on multi-graph data sets provide strong evidence for the superior power of the proposed approach in modeling cross-graph inference and large-scale optimization. 

\section*{Acknowledgements}
We thank the reviewers for their helpful comments.
This work is supported in part by the
National Science Foundation (NSF) under grants IIS-1216282, 1350364, 1546329.

\bibliography{references}

\begin{thebibliography}{19}
\providecommand{\natexlab}[1]{#1}
\providecommand{\url}[1]{\texttt{#1}}
\expandafter\ifx\csname urlstyle\endcsname\relax
  \providecommand{\doi}[1]{doi: #1}\else
  \providecommand{\doi}{doi: \begingroup \urlstyle{rm}\Url}\fi

\bibitem[Basilico \& Hofmann(2004)Basilico and Hofmann]{basilico2004unifying}
Basilico, Justin and Hofmann, Thomas.
\newblock Unifying collaborative and content-based filtering.
\newblock In \emph{Proceedings of the twenty-first international conference on
  Machine learning}, pp.\ ~9. ACM, 2004.

\bibitem[Ben-Hur \& Noble(2005)Ben-Hur and Noble]{ben2005kernel}
Ben-Hur, Asa and Noble, William~Stafford.
\newblock Kernel methods for predicting protein--protein interactions.
\newblock \emph{Bioinformatics}, 21\penalty0 (suppl 1):\penalty0 i38--i46,
  2005.

\bibitem[Bordes et~al.(2013)Bordes, Usunier, Garcia-Duran, Weston, and
  Yakhnenko]{bordes2013translating}
Bordes, Antoine, Usunier, Nicolas, Garcia-Duran, Alberto, Weston, Jason, and
  Yakhnenko, Oksana.
\newblock Translating embeddings for modeling multi-relational data.
\newblock In \emph{Advances in Neural Information Processing Systems}, pp.\
  2787--2795, 2013.

\bibitem[Cai et~al.(2011)Cai, He, Han, and Huang]{cai2011graph}
Cai, Deng, He, Xiaofei, Han, Jiawei, and Huang, Thomas~S.
\newblock Graph regularized nonnegative matrix factorization for data
  representation.
\newblock \emph{Pattern Analysis and Machine Intelligence, IEEE Transactions
  on}, 33\penalty0 (8):\penalty0 1548--1560, 2011.

\bibitem[Duchi et~al.(2011)Duchi, Hazan, and Singer]{duchi2011adaptive}
Duchi, John, Hazan, Elad, and Singer, Yoram.
\newblock Adaptive subgradient methods for online learning and stochastic
  optimization.
\newblock \emph{The Journal of Machine Learning Research}, 12:\penalty0
  2121--2159, 2011.

\bibitem[Fergus et~al.(2009)Fergus, Weiss, and Torralba]{fergus2009semi}
Fergus, Rob, Weiss, Yair, and Torralba, Antonio.
\newblock Semi-supervised learning in gigantic image collections.
\newblock In \emph{Advances in neural information processing systems}, pp.\
  522--530, 2009.

\bibitem[Hattori et~al.(2003)Hattori, Okuno, Goto, and
  Kanehisa]{hattori2003heuristics}
Hattori, Masahiro, Okuno, Yasushi, Goto, Susumu, and Kanehisa, Minoru.
\newblock Heuristics for chemical compound matching.
\newblock \emph{Genome Informatics}, 14:\penalty0 144--153, 2003.

\bibitem[Joachims(2002)]{joachims2002optimizing}
Joachims, Thorsten.
\newblock Optimizing search engines using clickthrough data.
\newblock In \emph{Proceedings of the eighth ACM SIGKDD international
  conference on Knowledge discovery and data mining}, pp.\  133--142. ACM,
  2002.

\bibitem[Kolda \& Bader(2009)Kolda and Bader]{kolda2009tensor}
Kolda, Tamara~G and Bader, Brett~W.
\newblock Tensor decompositions and applications.
\newblock \emph{SIAM review}, 51\penalty0 (3):\penalty0 455--500, 2009.

\bibitem[Liu \& Yang(2015)Liu and Yang]{liu2015bipartite}
Liu, Hanxiao and Yang, Yiming.
\newblock Bipartite edge prediction via transductive learning over product
  graphs.
\newblock In \emph{Proceedings of The 32nd International Conference on Machine
  Learning}, pp.\  1880--1888, 2015.

\bibitem[Narita et~al.(2012)Narita, Hayashi, Tomioka, and
  Kashima]{narita2012tensor}
Narita, Atsuhiro, Hayashi, Kohei, Tomioka, Ryota, and Kashima, Hisashi.
\newblock Tensor factorization using auxiliary information.
\newblock \emph{Data Mining and Knowledge Discovery}, 25\penalty0 (2):\penalty0
  298--324, 2012.

\bibitem[Nickel et~al.(2011)Nickel, Tresp, and Kriegel]{nickel2011three}
Nickel, Maximilian, Tresp, Volker, and Kriegel, Hans-Peter.
\newblock A three-way model for collective learning on multi-relational data.
\newblock In \emph{Proceedings of the 28th international conference on machine
  learning (ICML-11)}, pp.\  809--816, 2011.

\bibitem[Oseledets(2011)]{oseledets2011tensor}
Oseledets, Ivan~V.
\newblock Tensor-train decomposition.
\newblock \emph{SIAM Journal on Scientific Computing}, 33\penalty0
  (5):\penalty0 2295--2317, 2011.

\bibitem[Rendle et~al.(2009)Rendle, Balby~Marinho, Nanopoulos, and
  Schmidt-Thieme]{rendle2009learning}
Rendle, Steffen, Balby~Marinho, Leandro, Nanopoulos, Alexandros, and
  Schmidt-Thieme, Lars.
\newblock Learning optimal ranking with tensor factorization for tag
  recommendation.
\newblock In \emph{Proceedings of the 15th ACM SIGKDD international conference
  on Knowledge discovery and data mining}, pp.\  727--736. ACM, 2009.

\bibitem[Smith \& Waterman(1981)Smith and Waterman]{smith1981identification}
Smith, Temple~F and Waterman, Michael~S.
\newblock Identification of common molecular subsequences.
\newblock \emph{Journal of molecular biology}, 147\penalty0 (1):\penalty0
  195--197, 1981.

\bibitem[Tang et~al.(2008)Tang, Zhang, Yao, Li, Zhang, and
  Su]{tang2008arnetminer}
Tang, Jie, Zhang, Jing, Yao, Limin, Li, Juanzi, Zhang, Li, and Su, Zhong.
\newblock Arnetminer: extraction and mining of academic social networks.
\newblock In \emph{Proceedings of the 14th ACM SIGKDD international conference
  on Knowledge discovery and data mining}, pp.\  990--998. ACM, 2008.

\bibitem[Yamanishi et~al.(2008)Yamanishi, Araki, Gutteridge, Honda, and
  Kanehisa]{yamanishi2008prediction}
Yamanishi, Yoshihiro, Araki, Michihiro, Gutteridge, Alex, Honda, Wataru, and
  Kanehisa, Minoru.
\newblock Prediction of drug--target interaction networks from the integration
  of chemical and genomic spaces.
\newblock \emph{Bioinformatics}, 24\penalty0 (13):\penalty0 i232--i240, 2008.

\bibitem[Yu \& Chu(2008)Yu and Chu]{yu2008gaussian}
Yu, Kai and Chu, Wei.
\newblock Gaussian process models for link analysis and transfer learning.
\newblock In \emph{Advances in Neural Information Processing Systems}, pp.\
  1657--1664, 2008.

\bibitem[Zhu et~al.(2003)Zhu, Ghahramani, Lafferty, et~al.]{zhu2003semi}
Zhu, Xiaojin, Ghahramani, Zoubin, Lafferty, John, et~al.
\newblock Semi-supervised learning using gaussian fields and harmonic
  functions.
\newblock In \emph{ICML}, volume~3, pp.\  912--919, 2003.

\end{thebibliography}
\bibliographystyle{icml2016}

\end{document}